\documentclass[preprint,12pt]{elsarticle}
\usepackage{graphicx, amsmath, amssymb, amsthm,algorithmic}
\usepackage[ruled,vlined]{algorithm2e}

\newtheorem{thm}{Theorem}

\newtheorem{prop}{Proposition}

\newtheorem{defn}{Definition}

\journal{Bio Systems}

\begin{document}

\begin{frontmatter}

\title{Theoretical Formulation and Analysis of the Deterministic Dendritic Cell Algorithm}
\author[label1]{Feng Gu}
\author[label2]{Julie Greensmith and Uwe Aickelin}
\address[label1]{School of Computing, University of Leeds, LS2 9JT, UK}
\address[label2]{School of Computer Science, University of Nottingham, NG8 1BB, UK}

\begin{abstract}
As one of the emerging algorithms in the field of Artificial Immune Systems (AIS), the Dendritic Cell Algorithm (DCA) has been successfully applied to a number of challenging real-world problems. However, one criticism is the lack of a formal definition, which could result in ambiguity for understanding the algorithm. Moreover, previous investigations have mainly focused on its empirical aspects. Therefore, it is necessary to provide a formal definition of the algorithm, as well as to perform runtime analyses to reveal its theoretical aspects. In this paper, we define the deterministic version of the DCA, named the dDCA, using set theory and mathematical functions. Runtime analyses of the standard algorithm and the one with additional segmentation are performed. Our analysis suggests that the standard dDCA has a runtime complexity of $\mathcal{O}(n^{2})$ for the worst-case scenario, where $n$ is the number of input data instances. The introduction of segmentation changes the algorithm's worst case runtime complexity to $\mathcal{O}(\max(nN,nz))$, for DC population size $N$ with size of each segment $z$. Finally, two runtime variables of the algorithm are formulated based on the input data, to understand its runtime behaviour as guidelines for further development. 
\end{abstract}

\begin{keyword}
Artificial Immune Systems, Dendritic Cell Algorithm, Runtime Analysis, Formulation and Formalisation 


\end{keyword}

\end{frontmatter}

\section{Introduction}
Artificial Immune Systems (AIS)~\cite{aisci2002,timmis2008a} are computer systems inspired by both theoretical immunology and observed immune functions, principles and models, which are applied to real-world problems. The human immune system from which AIS draw inspiration, is evolved to protect the host from a wealth of invading micro-organisms. AIS are developed to provide similar defensive properties within a computing context. Initially, AIS were based on simple models of the human immune system. As noted by Stibor {\it et al.}~\cite{stibor2005}, `first generation' immune algorithms, such as negative selection and clonal selection, do not produce the same high-quality performance as the human immune system. These algorithms, negative selection in particular, are prone to problems with scalability and the generation of excessive false alarms, when used to solve problems such as network-based intrusion detection. Recent AIS use more rigorous and up-to-date immunology and are developed in collaboration with modellers and immunologists. The resulting algorithms are believed to encapsulate the desirable properties of immune systems, including robustness, error tolerance, and self-organisation~\cite{aisci2002}.

One such `second generation'  immune algorithms is the Dendritic Cell Algorithm (DCA) \cite{greensmith2007}. The algorithm is inspired by functions of the dendritic cells (DCs) of the innate immune system, while incorporating principles of a key novel theory in immunology, named the {\it danger theory}~\cite{manfredlutz2002}. An abstract model of natural DC behaviour is used as the foundation of the developed algorithm. The DCA has been successfully applied to numerous security-related problems, including port scan detection~\cite{greensmith2007}, botnet detection~\cite{aihammadi2008} and as a classifier for robot security~\cite{oates2007}. These applications refer to the area of anomaly detection, which is essentially one particular type of binary classification with an `anomalous' class and a `normal' class. According to results of these applications, the DCA has shown not only good performance in terms of detection rate, but also the ability to reduce the rate of false alarms in comparison to other systems, such as Self Organising Maps (SOM)~\cite{greensmith2008c}. 

However, there are also issues concerning the DCA. One criticism is the lack of a formal definition, which could result in ambiguity for understanding the algorithm and thus lead to incorrect applications and implementations. It is pointed out in~\cite{stibor2009} that the DCA shares similarities to linear classifiers since it employs a linear discriminant function for signal transformation. However, the DCA is not simply a collection of linear classifiers, as it performs classification based on the temporal correlation of a multi-agent DC population, rather than linear signal transformation. Signal transformation is used to identify if any anomalies occurred in the past. Whether the identified anomalies can be correctly correlated with potential causes is determined by the effectiveness of the temporal correlation performed at the population level. As a first step, a formal definition should be provided for presenting the algorithm in a clear and accessible manner. 

Previous investigations have mainly focused on its empirical aspects, evidenced by experimental results on a range of problem domains. Except for the geometry analysis of Stibor {\it et al.}~\cite{stibor2009} that was later extended in Oates's thesis~\cite{oate2010a}, theoretical analysis of the DCA has barely been performed, and most theoretical aspects of the algorithm have not yet been revealed. Other immune inspired algorithms, such as negative and clonal selection algorithms, were theoretically presented in~\cite{timmis2008}. Elberfeld and Textor~\cite{elberfeld2011} theoretically analysed string-based negative selection algorithms, to show the possibility of reducing the worst-case runtime complexity from exponential to polynomial, through compressing detectors. More recently, the work of Zarges~\cite{zarges2008,zarges2009} theoretically analysed one of the vital components of the clonal selection based algorithms, namely inversely proportional mutation rates. Jansen and Zarges~\cite{janson2011} performed a theoretical analysis of immune inspired somatic contiguous hypermutations for function optimisation. As a result, it is important to conduct a similar theoretical analysis of the DCA, to determine its runtime complexity and numerous other algorithmic properties, in line with other AIS. 

In this paper, we extend the work presented in~\cite{gu2009b}, which involved formal specifications of a single-cell model at the behavioural level using interval temporal logic~\cite{moszkowski1985}. Note the algorithm demonstrated in this work is the deterministic DCA (dDCA)~\cite{greensmith2008b}, created by removing stochastic components for the ease of analysis. Any statements of the DCA made subsequently are referred to the dDCA. The aim is to provide a clear and accessible definition of the DCA, as well as an initial theoretical analysis on the algorithm's runtime complexity and other algorithmic properties. As potential readers may not have a deep understanding of complicated formal methods such as the B-method~\cite{jeanraymond1996}, we use set theory and mathematical functions to specify the algorithm. From the formal definitions, theoretical analyses on the runtime complexity are performed, for the standard algorithm and an extended system with segmentation. Moreover, the formulations of two important runtime variables are included to present the algorithm's runtime behaviour, and to provide guidelines for future development. The paper is organised as follows, an overview of the DCA is given in Section 2, the formal definition is presented in Section 3, runtime analyses are shown in Section 4, formulation of two runtime variables is described in Section 5, and finally conclusions and future work are presented in Section 6. 

\section{The Dendritic Cell Algorithm}
\subsection{Biological Background}
The DCA is inspired by functions of the dendritic cells (DCs) of the innate immune system, which forms part of the body's first line of defence against invaders. DCs exhibit the ability to combine a multitude of molecular information and to interpret this information for the T-cells of the adaptive immune system. This could lead to the induction of various immune responses against perceived pathogenic threats. Therefore, DCs are often seen as detectors responsible for policing different tissues, as well as inductive mediators for a variety of immune responses. 

In general, two types of molecular information are processed by DCs, namely `signal' and `antigen'. Signals are collected by DCs from their local environment and consist of indicators of the health status of the monitored tissue. Throughout its lifespan, an individual DC will exist in one of three states, namely `immature', `semi-mature' and fully `mature', as shown in Figure~\ref{fig:dc_states}. In the initial immature state, DCs are exposed to a combination of signals, and perform phagocytosis to ingest substances from their surroundings. Based on the concentration of presented signals, DCs differentiate into either a `fully mature' form to activate the adaptive immune system, or a `semi-mature' form to suppress it. If a DC is exposed to a combination of signals generated from a healthy or steady state tissue environment, such as no occurrence of tissue damage, it more likely becomes a semi-mature DC. Conversely, if a DC is presented with a combination of signals generated from a damaged tissue environment, such as the presence of unregulated cell death, it more likely differentiates into a fully mature DC. Natural DCs bind to and process many cytokine signals. In an abstract model of DC behaviour developed by Greensmith~\cite{greensmith2007}, the following categories are defined. 
\begin{itemize}
\item {\bf PAMP}: Pathogenetic Associated Molecular Patterns, molecular signatures of pathogens which are recognised by Toll-Like Receptors (TLRs) on the surface of DCs, and they are highly influential to the transition from immature state to fully mature state;
\item {\bf Danger}: released by damaged tissue cells subject to necrosis (unregulated cell death), they have a lower effect than PAMPs on the maturation towards fully mature state; 
\item {\bf Safe} signals are derived from the cells that encounter apoptosis (programmed cells death), TNF-$\alpha$ (Tumour Necrosis Factor) is one candidate of safe signals, they contribute to the maturation from immature state to semi-mature state; 
\end{itemize}

During the immature state, DCs also collect debris in the tissues which are subsequently combined with the environmental signals. Some of the `suspicious' debris collected are known as antigens, and they are proteins originating from potential invading entities. DCs combine the `suspect' antigens with evidence in the form of signals to correctly instruct the adaptive immune system to respond, or become tolerant to the presented antigens. For more detailed information regarding the underlying biological mechanisms, please refer to~\cite{greensmith2007, manfredlutz2002}. 
\begin{figure} 
\begin{center}
\includegraphics[width=\textwidth]{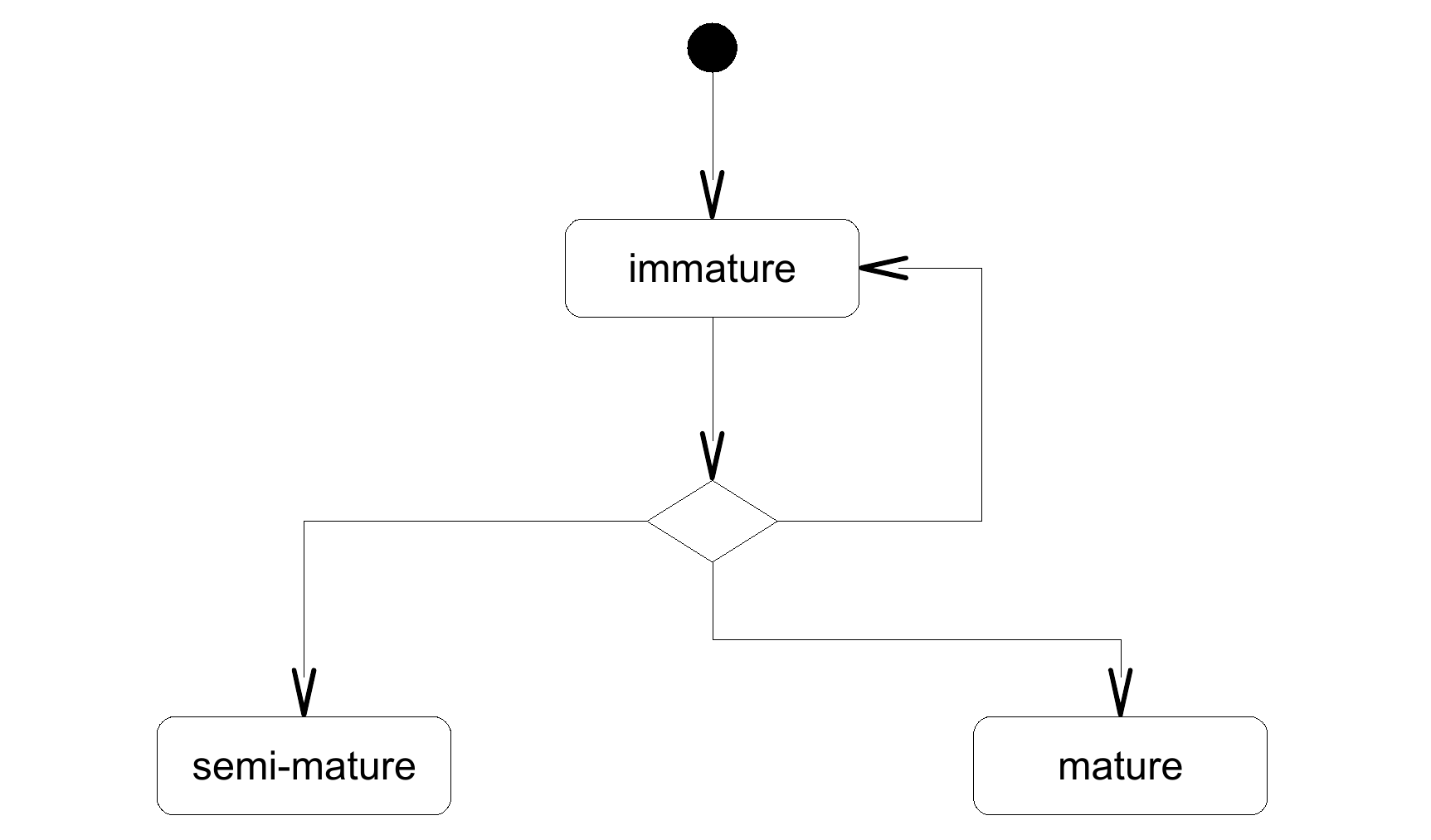}
\end{center}
\caption{A state-chart describing the three states of an individual DC.}
\label{fig:dc_states}
\end{figure}

\subsection{Algorithmic Details}
The DCA was designed and developed based on an abstract DC model created by Greensmith~\cite{greensmith2007}. It incorporates the functionality of DCs including data fusion, state differentiation and causal correlation. As per the natural system, there are two types of input data, namely `antigen' and `signal'. It is generally assumed that certain causal relationship exists between the two data streams. Antigens are categorical values that can be various states of a problem domain or the entities of interest associated with a monitored system. Signals are represented as vectors of real-valued numbers, and they are measures of a monitored system's status within certain time periods. In real-world applications, antigens represent what is to be classified within a given problem domain. For instance, they can be process IDs in computer security problems~\cite{aihammadi2008, greensmith2007a}, a small range of positions and orientations of robots~\cite{oates2007}, the proximity sensors of online robotic systems~\cite{mokhtar2009}, or the time stamps of records collected in biometric data~\cite{gu2009c}. Signals represent system context of a host or a measure of network traffic~\cite{aihammadi2008, greensmith2007a}, the readings of various sensors in robotic systems~\cite{oates2007, mokhtar2009}, or the biometric data captured from a monitored automobile driver~\cite{gu2009c}. Signals are normally pre-categorised as `PAMP',`Danger' or `Safe'. The semantics of these signal categories is listed as follows: 
\begin{itemize}
\item \textbf{PAMP}: increases in value as the observation of anomalous behaviour, it is a confidence indicator of anomaly, which usually is presented as signatures of the events that can definitely cause damage to the system; 
\item \textbf{Danger}: reflects to potential anomalies, as the values increases, the confidence of the abnormal status of the monitored system increases accordingly; 
\item \textbf{Safe}: increases in value in conjunction with observed normal behaviour, this is a confidence indicator of normal, predictable or steady-state system behaviour. 
\end{itemize}
Increases in the value of safe signal suppress the effect of the PAMP and Danger signals within the algorithm, as per what is observed in the natural system. This immunological property is incorporated within the DCA in the form of predefined weights for each signal category, for the transformation from input signals to output signals, which are `$CSM$' and `$K$' signals. The $CSM$ signal reflects the amount of information a DC has processed, i.e. when to make decisions, while the $K$ signal is a measure indicating the polarisation towards anomaly or normality, i.e. how to make decisions. The output signals are used to evaluate the status of the system monitored by the analysis component of the algorithm. Such a signal transformation process is displayed in Figure~\ref{fig:dca_sigs}.
\begin{figure}[t]
\begin{center}
\includegraphics[width=0.8\textwidth]{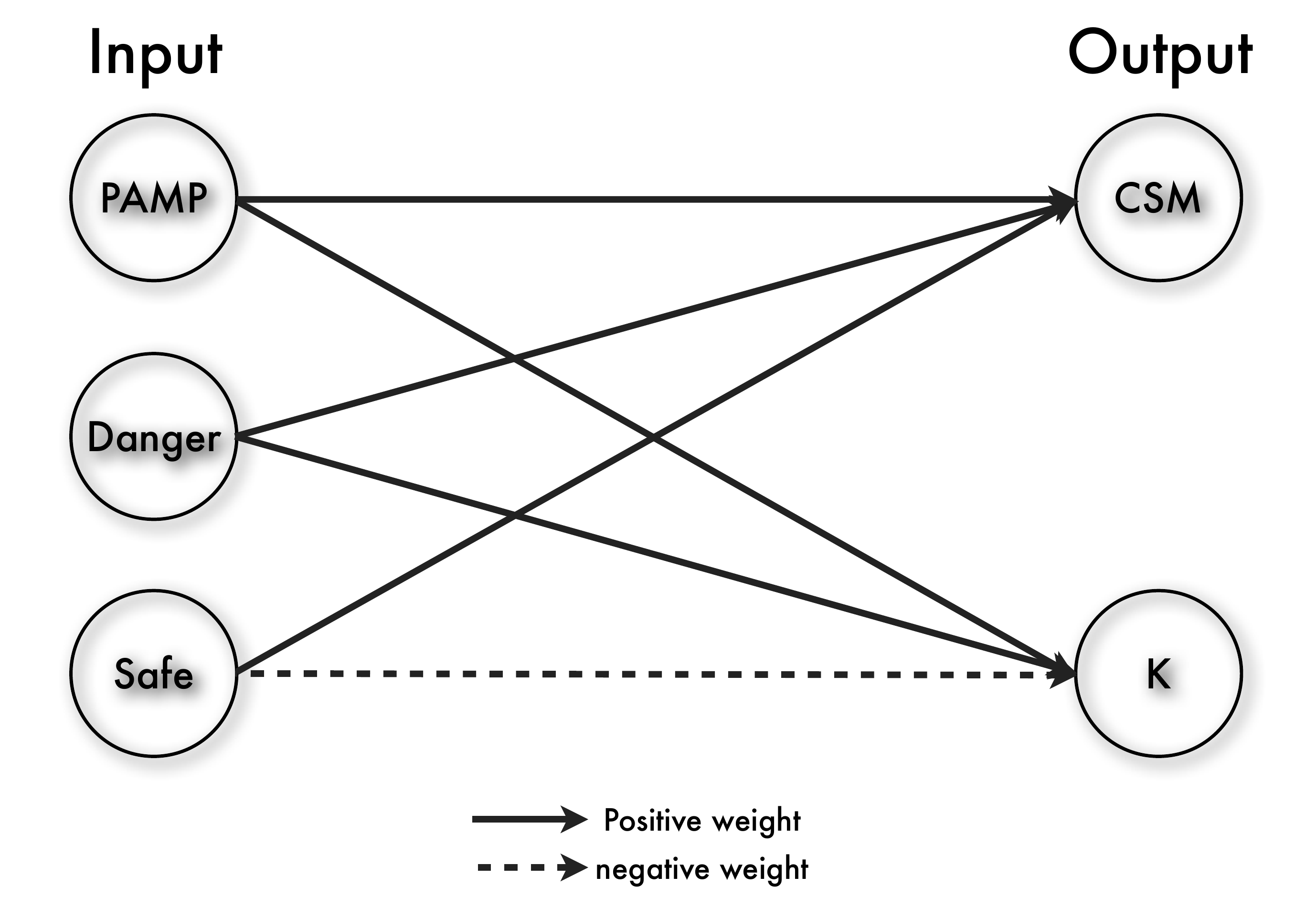}
\end{center}
\caption{An illustration of the signal transformation process of the DCA.}
\label{fig:dca_sigs}
\end{figure}  

In order to achieve its detection ability, the DCA initialises a population of artificial DCs operating in parallel as detectors. Each DC is given a distinct limit of its lifespan, which creates a dynamic time window effect in the population~\cite{oates2008a}. This leads to the same signal and antigen data streams being processed by every DC, during different time periods across the analysed time series. A temporal correlation between signals and antigens is also performed by each DC internally, to capture the causal relationship within the data. As suggested in~\cite{greensmith2008b}, to perform correct correlation, the signals are supposed to appear after the antigens, and the delay should be shorter than the time window created by each DC. 

During detection, each individual DC updates its antigen profile by storing the sampled antigens internally. In the meantime, the output signals produced by the signal transformation are accumulated, to update the DC's lifespan and signal profile. The DC's lifespan is subtracted by the cumulative {\it CSM}, which gives the difference between the amount of information initially allowed for a DC and that has been processed by the DC so far. Such difference reflects to if the DC has processed sufficient information and is ready to make decisions. On the other hand, its signal profile is added by the cumulative {\it K}, to aggregate the polarisation towards anomaly or normality indicated by its tendency toward $-\infty$ or $+\infty$. As soon as the DC's lifespan reaches zero, it stops performing signal transformation and temporal correlation. The association between the cumulative {\it K} and sampled antigens within the DC, termed `processed information', is then accumulated by the analysis phase to produce the final detection results. Once a matured DC has presented the processed information, it is reset to its default form. Here, the population size is generally kept constant, but can be user specified. The entire process of different steps of the DCA is illustrated in Figure~\ref{fig:dca_steps}. 
\begin{figure}[t]
\begin{center}
\includegraphics[width=\textwidth]{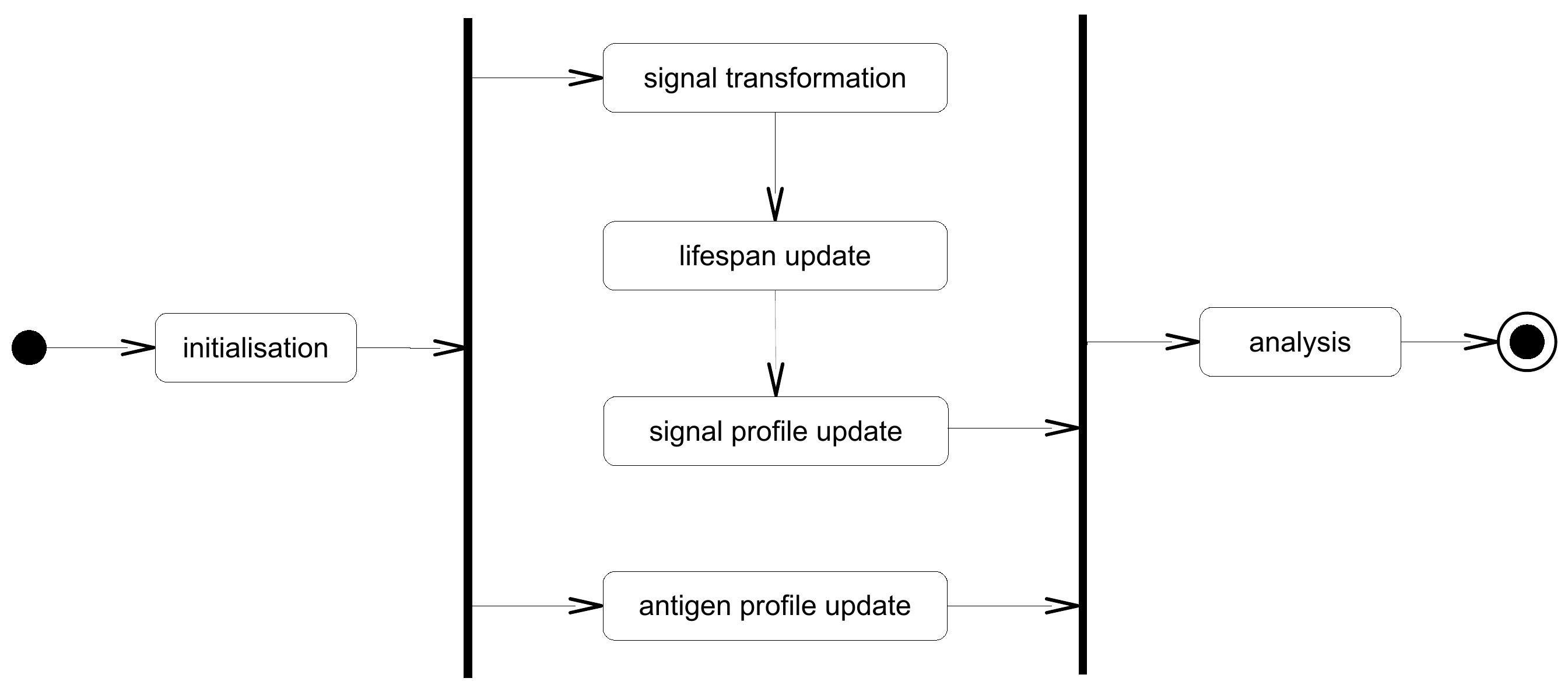}
\end{center}
\caption{An illustration of different steps of the DCA, where the initialisation and analysis steps are performed at the population level and the rest of the steps (bounded within the two vertical lines) are performed at the individual DC level.}
\label{fig:dca_steps}
\end{figure}

\section{Formalisation of the DCA}
In this section, we formally define data structures and procedural operations of the DCA at the population level. Unlike specifications of a single DC at the behavioural level in~\cite{gu2009b}, here we focus on specifying the entire DC population using quantitative measures at the functional level. Instead of using more advanced and possibly more complex interval temporal logic~\cite{moszkowski1985}, set theory and mathematical functions e.g. addition, multiplication and recursion are used for clarity. This aims to present the algorithm in a comprehensive way, which can be easily accessed by readers who may not be familiar with formal logic.  

\subsection{Data Structures}
Define $\mathsf{Signal}\subseteq\mathbb{R}^{m}$ and $\mathsf{Antigen}\subseteq\mathbb{N}$ as the two types of input data. Within a discrete time space $\mathsf{Time} = \{1,2,\ldots,t,\ldots\}$, the input data can be defined as a function $S: \mathsf{Time}\rightarrow \mathsf{Signal}\cup\mathsf{Antigen}$, and $S(t)$ is a data instance at a time point $t\in\mathsf{Time}$. Elements from $\mathsf{Signal}$ are input signal instances of the algorithm, and are represented as $m$-dimensional real-valued vectors. These are usually normalised into a non-negative range, e.g. $[0,1]$, as the input to the DCA. In many applications, $m=3$ is the standard case, corresponding to the three input signal categories of the DCA as described in Section 2. Elements from $\mathsf{Antigen}$ are categorical identifiers of certain objects to be classified, and are often represented as natural numbers starting from one, where the order is ignored.  

Define the weight matrix of signal transformation as
\begin{equation*}
\mathbf{W}= \left[ \begin{array}{ccc}
w_{11} & \cdots & w_{1m} \\
w_{21} & \cdots & w_{2m} \end{array} \right]
\end{equation*}
where each entry $w_{ij}\in\mathbb{R}$. The weight matrix $\mathbf{W}$ is used to transform the $m$-dimensional input signals to two categories of output signals, namely `$CSM$' and `$K$'. It is usually predefined by users and kept constant during runtime. Entries in the weight matrix are based on empirical results from the underlying immunology of natural DCs~\cite{greensmith2007}. 

Let $\mathsf{Population}$ be an index set of DCs and $N=|\mathsf{Population}|$ be the population size ($N=100$ is a popular choice). The index of a DC is $i\in\mathsf{Population}$. The function of assigning the initial lifespan to a DC is defined as $I : \mathsf{Population}\rightarrow\mathbb{R}$, where $I(i)\neq I(j)$ ($i\neq j\in\mathsf{Population}$). The function of initialising the antigen profile of the DC is defined as $M : \mathsf{Population}\rightarrow(a_{i1},a_{i2},...,a_{ik},...)$, where $(a_{i1},a_{i2},...,a_{ik},...)$ is a sequence storing the antigen instances sampled by a DC and $a_{ik}\in\mathsf{Antigen}$. The initial signal profile of a DC is usually set to zero.  

The output of each DC is stored as a pair $(a_{ik},r_{i})\in\mathsf{Antigen}\times \mathbb{R}$ in a list, where $r_{i}$ is the signal profile of a DC when it reaches a termination condition. We also define $\pi_{1}$ and $\pi_{2}$ as projection functions to obtain the first and second elements of a pair respectively. 

\subsection{Procedural Operations}
To access the data structures of the DCA, a series of one-step procedural operations are executed. Formally defining these operations is essential for the algorithm's runtime analysis. At the beginning ($t=1$), the algorithm initialises all the DCs indexed by $\mathsf{Population}$, through assigning the initial values of lifespans and signal profiles, named `{\bf DC initialisation}'. The value of $I(i)$ depends on the distribution function used to generate the initial lifespans of DCs. Both uniform distribution and Gaussian distribution can be applied to generate $I(i)$. The antigen profile of each DC is set as $\mathsf{Null}$ or empty, while the signal profile is set as zero. 

\begin{defn}[\textbf{signal transformation}] The signal transformation function $O: \mathsf{Time}\rightarrow \mathbb{R}\times \mathbb{R}$ is defined as
\begin{equation*}
O(t)=\left \{\begin{array}{l l}
\mathbf{W}^{\mathrm{T}}S(t), &~\mathrm{if}~S(t)\in \mathsf{Signal}; \\
\mathbf{0}, &~\mathrm{otherwise}. \end{array} \right. \\
\end{equation*}
\label{def:signal}
\textnormal{This operation is executed whenever $S(t)\in \mathsf{Signal}$ holds, and it performs the multiplication between a transposed $2\times m$ matrix and an $m$-dimensional vector  to produce a two dimensional vector of output signals, `$CSM$' and `$K$'. These are related to when and how to make decisions respectively. In the case that $S(t)\in \mathsf{Antigen}$, the function returns a zero vector.}
\end{defn}

\begin{defn}[\textbf{lifespan update}] The lifespan update function $F: \mathsf{Time}\times \mathsf{Population}\rightarrow \mathbb{R}$ is defined as
\begin{equation*}
F(t,i)=\left \{\begin{array}{l l}
I(i), &~\mathrm{if}~t=1; \\
I(i)-\pi_{1}(O(t)), &~\mathrm{if}~F(t-1,i)\le 0; \\
F(t-1,i)-\pi_{1}(O(t)), &~\mathrm{otherwise}. \end{array} \right.
\end{equation*}
\label{def:life}
\textnormal{When $t=1$, the initial value of $F$ is $I(i)$, which is the initial lifespan of the DC with an index $i$. It is repeatedly subtracted by $CSM$ signal until the termination condition, $F(t-1,i)\le 0$,  is reached. The function is then reset to `$I(i)-\pi_{1}(O(t))$' (not $I(i)$), due to the function $O(t)$ being executed at a regular basis, e.g. at every single time point $t$.}
\end{defn} 

\begin{defn}[\textbf{signal profile update}] The signal profile update function $G: \mathsf{Time}\times \mathsf{Population}\rightarrow \mathbb{R}$ is defined as
\begin{equation*}
G(t,i)=\left \{\begin{array}{l l}
0, &~\mathrm{if}~t=1; \\
\pi_{2}(O(t)), &~\mathrm{if}~F(t-1,i)\le 0; \\
G(t-1,i)+\pi_{2}(O(t)), &~\mathrm{otherwise}. \end{array} \right.
\end{equation*}
\textnormal{When $t=1$, the value of $G$ is zero, which is the initial signal profile of the DC with an index $i$. It is repeatedly increased by the $K$ signal until the termination condition is reached. The function is then reset to `$\pi_{2}(O(t))$' (not 0), due to the function $O(t)$ being executed at a regular basis, e.g. at every single time point $t\in\mathsf{Time}$.}
\end{defn}

\begin{defn}[\textbf{antigen profile update}] The antigen profile update function $H: \mathsf{Time}\times \mathsf{Population}\rightarrow(a_{i1},a_{i2},\ldots,a_{ik},\ldots)$ is defined as
\begin{equation*}
H(t,i)=\left \{\begin{array}{l l}
\emptyset, &~t=1; \\
(H(t-1,i),S(t)), &~\mathrm{if}~S(t)\in \mathsf{Antigen}~\mathrm{and}~t>1; \\ 
H(t-1,i) &~\mathrm{if}~S(t)\in \mathsf{Signal}~\mathrm{and}~t>1, \end{array} \right.
\end{equation*}
\textnormal{where $H$ is initially empty. As a new antigen instance arrives, it is sampled by the DC with an index $i$ and its antigen profile is updated until the termination condition is reached. This function merely appends a list to another, which can be done in constant time regardless of the length of the lists, and thus considered as one-step operation as well. It is performed individually by each DC and the index of the DC selected to sample an incoming $S(t)\in \mathsf{Antigen}$ is defined as $i\equiv\theta\mod{N}$ ($i$ is congruent with $\theta$ modulo $N$), where $\theta$ is the number of antigen instances up to time $t$. This is termed the `sequential sampling' rule.}
\end{defn}

\begin{defn}[\textbf{output record}] Let $r_{i}=G(t,i)~\mathrm{s.t.}~F(t-1,i)\le 0$ be the signal profile of a DC, and $L : \mathbb{N}\rightarrow\mathsf{Antigen}\times\mathbb{R}$ denote the function that maps an index $j\in\mathbb{N}$ to an element of the output list. The output record function is defined as
\begin{equation*}
L(j) = (a_{ik},r_{i})\quad\forall k
\end{equation*}
\textnormal{where $L(j)$ is the $j$th element of the list. The antigen profile often consists of multiple values while the signal profile only contains one single value in the DC with an index $i$. This function essentially enumerates all the possible pairs and appends them to the output list, where each of them is assigned an index $j$. The list is then used to produce the final detection results in the analysis phase of the DCA.}
\end{defn}

\begin{defn}[\textbf{antigen counter}] The antigen counter function $C: \mathbb{N}\times\mathsf{Antigen}\rightarrow \{0,1\}$ is defined as
\begin{equation*}
C(j,\alpha)=\left \{\begin{array}{l l}
1, &~\mathrm{if}~\pi_{1}(L(j))=\alpha; \\
0, &~\mathrm{otherwise}. \end{array} \right. \\
\end{equation*}
\end{defn}

\begin{defn}[\textbf{signal profile abstraction}] The signal profile abstraction function $R: \mathbb{N}\times \mathsf{Antigen}\rightarrow\mathbb{R}$ is defined as
\begin{equation*}
R(j,\alpha)=\left \{\begin{array}{l l}
\pi_{2}(L(j)), &~\mathrm{if}~\pi_{1}(L(j))=\alpha; \\
0, &~\mathrm{otherwise}. \end{array} \right.
\end{equation*}
\textnormal{In the two functions above, $\alpha\in\mathsf{Antigen}$ is an antigen type. The function $C$ is used to count the number of instances of antigen type $\alpha$, and the function $R$ is used to calculate the sum of all $K$ values associated with antigen type $\alpha$. These two operations are performed for every antigen type and involve scanning the sequence of $L(j)$ in its entirety.} 
\end{defn}

\begin{defn}[\textbf{anomaly metric calculation}] Given the number of input instances is equal to $n$, the anomaly metric calculation function is defined as.
\begin{align*}
& K(\alpha)=\frac{\gamma}{\beta}~\mathrm{with}~\beta=\sum_{j=1}^{n} C(j,\alpha)~\mathrm{and}~\gamma=\sum_{j=1}^{n} R(j,\alpha)
\end{align*}
\textnormal{As $\mathsf{Antigen}\neq\emptyset$ and $\alpha\in\mathsf{Antigen}$, the minimum number of antigen instances is equal to one, so is the minimum number of antigen types. Therefore, we have $\beta\ge 1$. A threshold $\varepsilon$ can be applied for further classification. The value of the threshold depends on the underlying characteristics of the dataset used. An antigen type $\alpha$ is classified as anomalous if $K({\alpha})>\varepsilon$, and normal otherwise.}
\end{defn}
 
\section{Analysis of Runtime Complexity}
\subsection{The Standard DCA}
By combining the procedural operations of the DCA with {\bf for}, {\bf while} loops or {\bf if} statements the algorithm can be presented as in Algorithm~\ref{alg:ddca}. Previous applications of the DCA have shown that the runtime of the algorithm is relatively short and the consumption of computation power is also low~\cite{greensmith2008b}. However, theoretical analysis of the runtime complexity of the DCA, given a set of input data, has not yet been performed. Runtime analysis involves calculating the number of primitive operations or steps executed by an algorithm~\cite{cormen2009}. The analysis is often based on asymptotic theory, and its aim is to theoretically show the runtime complexity of an algorithm as a function of increasing input size $n$.  

As mentioned previously, applications of the DCA are referred to the area of anomaly detection. In AIS, one popular anomaly detection algorithm is known as the negative selection algorithm, which was shown to have an exponential runtime complexity~\cite{timmis2008}. An attempt of reducing the worst-case runtime complexity from exponential to polynomial was reported in~\cite{elberfeld2011}, however this reduction is only applicable when the input feature space is bit strings instead of real numbers. Other popular anomaly detection algorithms are more or less derived from techniques in machine learning~\cite{chandola2009}, e.g. K-Nearest Neighbour (KNN) with a runtime complexity of $\mathcal{O}(nd)$~\cite{duda2000}, decision trees algorithms with an exponential runtime complexity~\cite{duda2000}, and support vector machines (SVM) with a runtime complexity of $\mathcal{O}(n^{2}d)$~\cite{burges1998}, where $n$ is the number of input instances and $d$ is the dimensionality. As a result, the subsequent runtime analysis of the DCA reveals if the algorithm is competitive against other state-of-the-art anomaly detection algorithms. 

\begin{algorithm}[t]
\SetAlgoLined
\LinesNumbered
\SetKwInOut{Input}{input}
\SetKwInOut{Output}{output}
\Input{input data $S(t)$}
\Output{anomaly metric $K({\alpha})$}
\BlankLine
\ForEach(\tcc*[f]{Initialisation phase}){DC}{
DC initialisation\;
}
\While(\tcc*[f]{Detection phase}){input data}{ 
\If{antigen}{
select a DC $i$\; 
do $H(t, i)$\;
}
\If{signal}{
do $O(t)$\;
\ForEach{DC}{
do $F(t,i)$\;
do $G(t,i)$\;
\If{$F(t-1,i)\le 0$}{
do $L(j)$\;
}
}
}
}
\While(\tcc*[f]{Analysis phase}){output list}{
\ForEach{antigen type}{
do $C(j, \alpha)$\;
do $R(j, \alpha) $\;
do $K(\alpha)$\;
}
}
\BlankLine
\caption{Pseudocode of the DCA implementation, the selection of a DC when an antigen instance is presented is performed according to the `sequential sampling' rule. \label{alg:ddca}}
\end{algorithm}

Let $a$ be the number of antigen instances within the input data, $b=|\mathsf{Antigen}|$ be the number of antigen types and $N$ be the size of the DC population. According to previous applications~\cite{aihammadi2008, oates2007, greensmith2007a, mokhtar2009, gu2009c}, $N$ is usually user defined and independent of the increase of data size $n$. However, we often assume that $1\le N\le n$. In order to make the following analyses more general, the population size $N$ is considered a parameter of the algorithm. As the type of input data instances is either $\mathsf{Antigen}$ or $\mathsf{Signal}$, if the number of antigen instances is equal to $a$, the number of signal instance is $n-a$.  For the ease of analysis, the algorithm is divided into three phases as follows: 
\begin{enumerate}
\item {\tt Initialisation phase} - Line 1 to Line 3;
\item {\tt Detection phase} - Line 4 to Line 19;
\item {\tt Analysis phase} - Line 20 to Line 26.
\end{enumerate}

The calculation of runtime is performed phase by phase. Let $T_{1}(n)$, $T_{2}(n)$ and $T_{3}(n)$ be the runtime  of each phase respectively, and $T(n)=T_{1}(n) + T_{2}(n) + T_{3}(n)$ is the overall runtime  of the algorithm. Details of all the primitive operations of the algorithm are listed in Table~\ref{tab:priops}, including the {\bf line number} and the {\bf description} of each operation as well as the {\bf number of times} an operation is executed, corresponding to Algorithm~\ref{alg:ddca}. 
\begin{table}
\begin{center}
\begin{tabular}{c|l|l}
\hline
{\bf Line No.} & {\bf Description} & {\bf Times} \\ \hline\hline
1 & {\bf for} loop & $N$ \\ 
2 & DC initialisation & $N$ \\ 
4 & {\bf while} loop  & $n$ \\
5 & {\bf if} statement & $a$ \\
6 & select a DC $i$ & $a$ \\
7 & antigen profile update ($H(t, i)$) & $a$ \\
9 & {\bf if} statement  & $n-a$ \\
10 & signal transformation ($O(t)$) & $n-a$ \\
11 & {\bf for} loop &  $(n-a)\times N$ \\
12 & lifespan update ($F(t,i)$) & $(n-a)\times N$ \\
13 & signal profile update ($G(t,i)$) & $(n-a)\times N$ \\
14 & {\bf if} statement & $(n-a)\times N$ \\
15 & output record ($L(j)$) & $(n-a)\times N$ \\
20 & {\bf while} loop & $a$ \\
21 & {\bf for} loop & $a\times b$ \\
22 & antigen counter ($C(j, \alpha)$) & $a\times b$ \\
23 & signal profile abstraction ($R(j, \alpha)$) & $a\times b$ \\
24 & anomaly metric calculation ($K(\alpha)$) & $a\times b$ \\
\hline
\end{tabular}
\caption{Details of primitive operations of Algorithm~\ref{alg:ddca}, where $N$ is the size of DC population, $n$ is the data size, $a$ is the number of antigen instances, and $b$ is the number of antigen types.}
\label{tab:priops}
\end{center}
\end{table}

The {\tt initialisation phase} is only executed once for the entire DC population at the commencement of the algorithm. Its runtime is independent of the number of input instances $n$, but is determined by the population size $N$. Therefore, the runtime  of the {\tt initialisation phase} is calculated as follows. 
\begin{equation*}
T_{1}(n) = N + N = \mathcal{O}(N)
\end{equation*}

The runtime  of the {\tt detection phase} depends on the data size $n$, the number of antigen instances $a$, the number of signal instances $n-a$ and the size of the DC population $N$. Thus the runtime of the {\tt detection phase} is calculated as follows.
\begin{align*}
& T_{2}(n) = n + 3a + 2(n-a) + 5(n-a)N = 3n + 5N(n-a) + a\\
\Rightarrow\quad & \{a\le n\} \\
& T_{2}(n) = \mathcal{O}(n) + \mathcal{O}(N(n-a)) + \mathcal{O}(a) = \mathcal{O}(nN)
\end{align*}

The runtime of the {\tt analysis phase} is dependent on the size of the output list that is equal to the number of antigen instances $a$ and the number of antigen types $b$. The value of $b$ is determined by the number of states or entities to classify within a problem domain. Here we merely focus on the worst-case scenario, which occurs if $b=a$, and the number of antigen types is equal to the number of antigen instances. Therefore, we have $1\le b\le a\le n$. The runtime  of the {\tt analysis phase} is thus calculated as follows. 
\begin{equation*}
T_{3}(n) = a + ab + 3ab = \mathcal{O}(n^{2})
\end{equation*}

\begin{thm}
The runtime complexity of the standard DCA is bounded by $\mathcal{O}(n^{2})$, with respect to the data size $n$.
\label{thm:run}
\end{thm}

\begin{proof}[Proof]
\begin{align*}
& T(n) = T_{1}(n) + T_{2}(n) + T_{3}(n) \\
\Rightarrow\quad & \{T_{1}(n) = \mathcal{O}(N),~T_{2}(n)=\mathcal{O}(nN),~\mathrm{and}~T_{3}(n)=\mathcal{O}(n^{2})\} \\
& T(n) = \mathcal{O}(N) + \mathcal{O}(nN) + \mathcal{O}(n^{2}) \\
\Rightarrow\quad & \{1\le N\le n\} \\
& T(n)=\mathcal{O}(n^{2})
\end{align*}
Bounds provided by $\mathcal{O}$-notation are asymptotically tight.
\end{proof}

As suggested by Theorem~\ref{thm:run}, the DCA has a worst case runtime complexity of $\mathcal{O}(n^{2})$, which is quadratic. As a result, the DCA is indeed competitive in terms of processing large-sized datasets while keeping the runtime complexity under control, when compared to state-of-the-art anomaly detection algorithms. According to previous applications~\cite{aihammadi2008, greensmith2007a,oates2007,mokhtar2009,gu2009c, gu2008}, we often have $N\ll n$. Such a premise makes the runtime complexity of algorithm's initialisation and detection phases overall linear, while the analysis phase stays quadratic. This leads to the following work of modifying the analysis phase of the algorithm via an introduction of segmentation. 

\subsection{The DCA with Segmentation}
Segmentation is introduced to adapt the algorithm to online analysis~\cite{gu2009a}. Instead of analysing the processed information in a single operation at the termination of the detection phase, the output list is partitioned into smaller segments and the analysis is performed within each segment. We postulate that segmentation could potentially generate finer grained results, as well as performing analysis in parallel with the detection process. Here, we focus on the antigen based segmentation approach, as it is more favourable in actual applications~\cite{gu2009a}. One may think that the system with segmentation produces the final detection results much faster, as the analysis is performed during detection on a much smaller chunk of processed information. Based on the analysis of the standard DCA, it is possible to theoretically analyse the effect of segmentation on the algorithm's runtime complexity. Let $z$ be a predefined segment size and $1\le z\le n$. A segment is generated once the size of the output list reaches $z$, and an analysis on the current batch of processed information in the output list is performed. 

As a post-processing mechanism, segmentation only affects the {\tt analysis phase} of the algorithm, but not the {\tt initialisation phase} or {\tt detection phase}. The search space of the analysis of a segment is determined by the value of $z$. The number of segments created is equal to $\lceil n/z\rceil$, and they are indexed by $\{1,2,\ldots,k,\ldots,\lceil n/z\rceil\}$. Let $a_{k}\le z$ and $b_{k}\le z$ denote the number antigen instances and the number of antigen types in the $k$th segment respectively. As a result, the runtime complexity of each segment at the analysis phase is $T^{k}_{3}(n) = a_{k}b_{k}\le z^{2} = \mathcal{O}(z^{2})$.

\begin{thm}
The runtime complexity of the DCA with segmentation is bounded by $\mathcal{O}(\max(nN,nz))$, with respect to the data size $n$, the DC population size $N$, and the segment size $z$.
\label{thm:seg}
\end{thm}

\begin{proof}[Proof]
\begin{align*}
\Rightarrow\quad &\{T_{1}(n)=\mathcal{O}(N)~\mathrm{and}~T_{2}(n)=\mathcal{O}(nN)\} \\
& T(n) = \mathcal{O}(N) + \mathcal{O}(nN) + \sum_{k=1}^{\lceil n/z\rceil}a_{k}b_{k}\le\mathcal{O}(N) + \mathcal{O}(nN)+\lceil n/z\rceil\mathcal{O}(z^{2}) \\
\Rightarrow\quad & \{1\le N\le n~\mathrm{and}~1\le z\le n\} \\
& T(n) = \mathcal{O}(N) + \mathcal{O}(nN) + \mathcal{O}(nz) = \mathcal{O}(\max(nN,nz))
\end{align*}
\end{proof}

As shown in Theorem~\ref{thm:seg}, the introduction of segmentation changes the overall runtime complexity of the algorithm to $\mathcal{O}(\max(nN,nz))$. Depending on the values of $N$ and $z$, the runtime complexity can be either quadratic ($N=n\vee z=n$) or linear ($N\ll n\wedge z\ll n$). This is very attractive for online detection tasks, as it provides a means of online analysis that continuously and periodically produces results during detection. Additionally, the DCA with segmentation produces significantly different and better results than the standard version~\cite{gu2009a}. Therefore, segmentation is an important and necessary addition to the DCA from a practical point of view. Thus far only static segmentation with a fixed segment size has been applied to the DCA. The effect of variable segment sizes on the detection performance still requires further investigation. 

\section{Formulation of Runtime Properties}
Two runtime variables of the DCA are assessed, as they can be used as quantitative indicators of the changes to the algorithm's runtime behaviour. They are the number of matured DCs (those which reach the termination condition and are reset) and the number of processed antigens respectively. The number of matured DCs indicates that the amount of processed information is related to signal instances. Conversely, the number of processed antigens implies that the amount of processed information is related to antigen instances. In this section, the formulation of the above properties is given, to build up the mathematical foundation of the algorithm. This is obtained with respect to a time interval $[t_{b},t_{e}] := \{t_{b},t_{b}+1,t_{b}+2,\ldots,t_{e}\}\subseteq\mathsf{Time}$. 

\subsection{Number of Matured DCs}
The number of matured DCs within a time interval is related to the reset frequency of the DC population, which indicates the work-load of the DC population. This can be used to determine whether the current setup of the current system should be altered. If the frequency of DC resetting is too high, most of the DCs become matured and get reset before they acquire a sufficient amount of information. As a result, the range of lifespans of the DC population should be extended, allowing more information to be obtained. In conduction with extending the range of lifespans of the DC population, it is necessary to also increase the size of the DC population, so that the lifespans do not become sparse. 

This becomes crucial if the system is deployed online, as an online system is often required to perform continuous detection and adapt to the changes of real-time situations. The number of matured DCs in the DC population depends on the distribution function used for the generation of DC lifespans, in addition to the input data within the time interval of interest. To make the analysis manageable, two types of distributions for generating the initial DC lifespans are considered, namely uniform distribution~\cite{atkinson2004} and Gaussian distribution~\cite{atkinson2004}. The calculations will be done through using the mean value of lifespans of the DC population and the mean value of $CSM$ signals corresponding to all the input signal instances. They focus on the average number of matured DCs within a given time interval rather than the particular number per iteration. However, as the time interval is reduced, e.g. to the duration of one iteration, the two numbers could become approximate to each other.

\begin{prop}[uniform distribution]
If the lifespans of the DC population are generated from an arithmetic series $x_{i}=x_{1}+(i-1)d$, where $x_{i}$ is the n$\mathrm{th}$ element, $x_{1}$ is the first element and $d$ is the interval between two successive elements, the number of matured DCs in the DC population $\delta$ can be calculated as follows.
\begin{equation*}
\delta = \left\lfloor \frac{N\sum_{t=t_{b}}^{t_{e}}\pi_{1}(O(t))}{(t_{e}-t_{b})(x_{1}+\frac{N-1}{2}d)}\right\rfloor
\end{equation*}
\label{prop:mrarth}
\end{prop}

By default, the ascending order of lifespans of the DC population corresponds to the order of its indices. As a result, if the size of the DC population is equal to $N$, the lifespan of the last DC with an index $i=N$ is given as $x_{N}=x_{1}+(N-1)d$. As demonstrated in Section 3, the termination condition where a DC matures as soon as its lifespan reaches zero through subtracting the $CSM$ signals. 

\begin{proof}[Proof]
\begin{align*}
\Rightarrow\quad & \{\varphi = \frac{\sum_{t=t_{b}}^{t_{e}}\pi_{1}(O(t))}{t_{e}-t_{b}}~\mathrm{and}~\mu_{1} = \frac{x_{1}+x_{N}}{2}=x_{1}+\frac{N-1}{2}d\} \\
& \delta = \left\lfloor \frac{N\varphi}{\mu_{1}}\right\rfloor =\left\lfloor \frac{N\sum_{t=t_{b}}^{t_{e}}\pi_{1}(O(t))}{(t_{e}-t_{b})(x_{1}+\frac{N-1}{2}d)}\right\rfloor
\end{align*}
Where $\varphi$ is the mean value of the {\it CSM} signals within the interval $[t_{b},t_{e}]$ and $\mu_{1}$ is the mean lifespan of the DC population. 
\end{proof}

Uniform distribution is used in the dDCA~\cite{greensmith2008b} to generate the initial lifespans of the DC population. This produces a set of values that are uniformly distributed within a certain range. According to Proposition~\ref{prop:mrarth}, if the parameters (first element $x_{1}$ and the interval $d$) of the arithmetic series are given, the number of matured DCs within the time interval $[t_{b},t_{e}]$ can be calculated accordingly.  

\begin{prop}[Gaussian distribution]
If the lifespans of the DC population are generated from a Gaussian distribution $x\sim\mathcal{N}(\mu,\sigma^{2})$, then the following formula holds.
\begin{equation*}
\mathrm{Pr}\left(\left\lfloor\frac{N\sum_{t=t_{b}}^{t_{e}}\pi_{1}(O(t))}{(\mu-\frac{2\sigma}{\sqrt{N}})(t_{e}-t_{b})}\right\rfloor \le\delta\le \left\lfloor\frac{N\sum_{t=t_{b}}^{t_{e}}\pi_{1}(O(t))}{(\mu+\frac{2\sigma}{\sqrt{N}})(t_{e}-t_{b})}\right\rfloor\right)=0.95
\end{equation*}
\label{prop:mrnorm}
\end{prop}

\begin{proof}[Proof]
\begin{align*}
\Rightarrow\quad& \{\varphi = \frac{\sum_{t=t_{b}}^{t_{e}}\pi_{1}(O(t))}{t_{e}-t_{b}}~\mathrm{and}~\mu_{2}\sim\mathcal{N}(\mu,\frac{\sigma^{2}}{N})\} \\
& \mathrm{Pr}\left(\mu-2\frac{\sigma}{\sqrt{N}}\le \mu_{2}\le\mu+2\frac{\sigma}{\sqrt{N}}\right)=0.95 \\
\Rightarrow\quad & \{\delta = \left\lfloor \frac{N\varphi}{\mu_{2}}\right\rfloor = \left\lfloor \frac{N}{\mu_{2}(t_{e}-t_{b})}\sum_{t=t_{b}}^{t_{e}}\pi_{1}(O(t))\right\rfloor\} \\
& \left\lfloor\frac{N\sum_{t=t_{b}}^{t_{e}}\pi_{1}(O(t))}{(\mu-\frac{2\sigma}{\sqrt{N}})(t_{e}-t_{b})}\right\rfloor \le\delta\le \left\lfloor\frac{N\sum_{t=t_{b}}^{t_{e}}\pi_{1}(O(t))}{(\mu+\frac{2\sigma}{\sqrt{N}})(t_{e}-t_{b})}\right\rfloor
\end{align*}
$\mathrm{Pr}(\cdot)$ is the probability operator. If the sample size is $N$, the sample mean $\mu_{2}$ is bounded by a Gaussian distribution $x\sim\mathcal{N}(\mu,\frac{\sigma^{2}}{N})$~\cite{atkinson2004}. The lower and upper bounds of the sample mean can be used to induce the bounds of the number of matured DCs. 
\end{proof}

In practice, Gaussian distribution has not been used for generating the lifespans of the DC population, but it has been of great interest~\cite{oates2010} and would be a priority of future investigation. According to Proposition~\ref{prop:mrnorm}, if we know the mean ($\mu$) and variance ($\sigma^{2}$) of the Gaussian distribution from which the lifespans of the DC population are generated, the size of DC population $N$, and the input data instances within the time interval $[t_{b},t_{e}]$, we can show that there is a $0.95$ chance the number of matured DCs is bounded by the lower and upper bounds. This could provide sufficient information for adjusting the system according to real-time scenarios. 

\subsection{Number of Processed Antigens}
As demonstrated in~\cite{gu2009a}, segmentation is effective for maintaining or even improving detection accuracy on large-sized datasets. This may be due to the fact that the number of processed antigens could determine whether an analysis of the current batch of processed information is required. Different from input antigen instances, processed antigens are those, presented by matured DCs. Investigation of the relationship between the number of processed antigens and the input data becomes essential for understanding the DCA, as well as for the development of integrating segmentation with the algorithm. Additionally, a priori knowledge of the number of processed antigens, based on the input data, may facilitate choosing an appropriate segment size. Here, we focus on formulating the relationship between the number of processed antigens and the input data, in particular, the number of input antigens. Let $\theta\in\mathbb{N}$ at $t\in\mathsf{Time}$ be the number of input antigens that are fed into the system, and $\delta$ be the mean lifespan of the DC population. Similar to Proposition~\ref{prop:mrarth} or Proposition~\ref{prop:mrnorm}, the calculations focus on the average number of processed antigens within a given time interval rather than the particular duration per iteration. As the time interval decreases, e.g. to the duration of one iteration, the two numbers could also be approximate to each other. 

The method of calculating the number of processed antigens within a given time interval $[t_{b},t_{e}]$ should be introduced first. It is similar to placing balls into a number of bins that are ordered based on their indexes in a sequential manner. Placing starts from the first bin then the second bin and so forth. If we reach the last bin, the process starts over again. In the end, a number of bins, starting from the first one, are taken and the number of balls contained is counted. The balls are equivalent to input antigens, the bins are equivalent to DCs, and the action of counting the number of balls is equivalent to the action of counting the number of processed antigens. Proposition~\ref{prop:noags} formulates the relationship between the number of processed antigens and the input data in two cases. 
\begin{prop}[\textbf{number of processed antigens}] Let $\nu$ be the number of processed antigens within a given interval $[t_{b},t_{e}]$, $c\equiv\delta\mod{N}$ and $d\equiv\theta\mod{N}$, the following formula of $\nu$ holds.
\begin{align*}
\nu=\left\{\begin{array}{l l}
\left(\delta-N\left\lfloor\frac{\delta}{N}\right\rfloor)(1+\left\lfloor\frac{\theta}{N}\right\rfloor\right), &~\mathrm{if}~c<d; \\
\left(\delta-N-N\left\lfloor\frac{\delta}{N}\right\rfloor\right)\left\lfloor\frac{\theta}{N}\right\rfloor+\theta, &~\mathrm{otherwise}. \end{array} \right.
\end{align*}
\label{prop:noags}
\end{prop}

\begin{proof}[Proof]
\begin{align*}
& \{\mathrm{transform~modulus~to~floor~functions}\} \\
\Rightarrow\quad & c=\delta-N\left\lfloor\frac{\delta}{N}\right\rfloor~\mathrm{and}~d=\theta-N\left\lfloor\frac{\theta}{N}\right\rfloor \\
\Rightarrow\quad & \{\mathrm{sequential~sampling}\} \\
& \mathrm{Case~}1: c<d \\
& \nu=c\left\lfloor\frac{\theta}{N}\right\rfloor+c=\left(\delta-N\left\lfloor\frac{\delta}{N}\right\rfloor\right)\left(1+\left\lfloor\frac{\theta}{N}\right\rfloor\right) \\
& \mathrm{Case~}2: c\ge d \\
& \nu=c\left\lfloor\frac{\theta}{N}\right\rfloor+d=\left(\delta-N-N\left\lfloor\frac{\delta}{N}\right\rfloor\right)\left\lfloor\frac{\theta}{N}\right\rfloor+\theta
\end{align*}
The number of antigens sampled by each DC is determined by $\theta\mod{N}$, but as only matured DCs present processed antigens, the number of processed antigens is determined by $\delta$ and the maximum of $c$ and $d$.
\end{proof}

The formulas for the number of processed antigens have two cases, depending on the relationship between $c\equiv\delta\mod N$ and $d\equiv\theta\mod N$. These formulas can relate the runtime variables of the algorithm to the input data, without actually running the algorithm. This provides theoretical insights into tuning the algorithm for a given problem. 

\section{Conclusions and Future Work}
In this paper, we provide formal definitions of the data structures and procedural operations of the deterministic version of the DCA, name the dDCA. It aims to clearly present the algorithm, to prevent future misunderstanding and ambiguity that could result in inappropriate applications and implementations. Based on the formal definitions, a runtime analysis of the standard DCA is performed. The DCA achieves the the worst-case runtime complexity bounded by $\mathcal{O}(n^{2})$, which is quadratic. The analysis of the system with segmentation is also performed. We have shown that the introduction of segmentation does change the algorithm's runtime complexity and in certain cases it approximates to linear. In addition, it provides a means of performing continuous and periodic analysis for the DCA. This makes the algorithm very attractive for online detection tasks.  

Moreover, two runtime variables are formulated, the number of matured DCs and the number of processed antigens. This shows how the algorithm behaves within a given time interval based on the input data without actually running the algorithm. As a result, the formulas of two runtime variables can be used as the indicators of adjusting the setup of the system according to real-time situations during detection. This an important step for understanding some of the potentially beneficial properties of the algorithm from a theoretical perspective, which could facilitate further investigations on the usefulness of these properties with respect to anomaly detection problems. 

This work gives application independent insights to the algorithm, which can be used as guidelines for future development. One of the goals of future development of the DCA is to turn it into an automated and adaptive online detection system, and such a system has certain requirements to fulfil. Firstly, the system has to be computationally efficient. The analysis of the runtime complexity of the DCA shows even in worst case scenarios its runtime complexity is competitive against other popular anomaly detection algorithms. Secondly, the system should be able to adapt to real-time scenarios encountered during detection. This requires the insights of how the algorithm behaves during runtime, which can be assessed from the two runtime variables. As a result, new components can be developed and integrated within the algorithm to adjust the system based on the assessment of these two runtime variables. 

In terms of future work, the specifications can be further simplified and the algorithm can be presented using functional programming approach~\cite{cousineau1998}, to reveal more algorithmic details. In addition, synthetic datasets generated from various probability density functions can be used to test the formulas defined in this paper. We can also investigate other properties of the algorithm, for example, the moving window effect created by each DC and the relationship between the size of DC population and the detection performance. Different methods of generating the initial lifespans of the DC population should also be investigated, in addition to the relationship between the weight matrix and the detection performance. 

\begin{table}
\begin{center}
\begin{tabular}{c|l|l}
\hline
{\bf Page No.} & {\bf Notation} & {\bf Description} \\ \hline\hline
9 & $\mathsf{Signal}$ & a set of signal instances \\ 
9 & $\mathsf{Antigen}$ & a set of antigen instances \\ 
9 & $t$ & a time point \\
9 & $S(t)$ & a map of $t$ to an input instance \\
9 & $\mathbf{W}$ & weight matrix of signal transformation \\
9 & $N$ & DC population size \\
9 & $I$ & an index set of DCs \\
9 & $\pi_{1}$ & projection function for the first element \\
9 & $\pi_{2}$ & projection function for the second element \\
10 & $O(t)$  & signal transformation function \\
10 & $F(t,i)$ & lifespan update function \\
10 & $G(t,i)$ & signal profile update function \\
11 & $H(t,i)$ & antigen profile update function \\
11 & $L(j)$ & output record function \\
11 & $C(j,\alpha)$ & antigen counter function \\
12 & $R(j,\alpha)$ & signal profile abstraction function \\
12 & $K(\alpha)$ & anomaly metric calculation function \\
12 & $n$ & size of input data \\
13 & $a$ & number of antigen instances \\
13 & $b$ & number of antigen types \\
15 & $z$ & segment size \\
\hline
\end{tabular}
\caption{List of terms and definitions used in Section 3 and Section 4.}
\label{tab:lsterms}
\end{center}
\end{table}

%\section*{Acknowledgement}
%The authors would like to thank Jan Feyereisl and Thomas Jansen for their valuable comments of this paper.

\bibliographystyle{plain}
\bibliography{mybib}

\end{document}